\documentclass[twoside]{article}
\usepackage[accepted]{aistats2015}
\usepackage[authoryear]{natbib}
\usepackage{hyperref}
\usepackage{url}
\usepackage{graphicx} 
\usepackage{subfigure} 
\usepackage{algorithm}
\usepackage{algorithmic}
\usepackage{amsmath}[11pt]
\usepackage{amsmath, amsthm, amssymb}
\usepackage{amsthm}
\usepackage{amsfonts}[11pt]
\usepackage{mathtools}
\usepackage{bbm}

\newtheorem{thm}{Theorem}
\newtheorem{cor}[thm]{Corollary}
\newtheorem{lem}[thm]{Lemma}
\newcommand{\bc}{\begin{center}}
\newcommand{\ec}{\end{center}}
\newcommand{\bq}{\begin{quote}}
\newcommand{\eq}{\end{quote}}

\newcommand{\be}{\begin{equation}}
\newcommand{\ee}{\end{equation}}
\newcommand{\beqa}{\begin{eqnarray*}}
\newcommand{\eeqa}{\end{eqnarray*}}
\newcommand{\beqn}{\begin{eqnarray}}
\newcommand{\eeqn}{\end{eqnarray}}
\newcommand{\bbibl}{\begin{thebibliography}{99}}
\newcommand{\ebibl}{\end{thebibliography}}

\newcommand{\ba}{\begin{array}}
\newcommand{\ea}{\end{array}}

\DeclareMathOperator*{\argmax}{argmax}
\DeclareMathOperator*{\argmin}{argmin}

\newcommand{\E}{\mathbb{E}}
\DeclarePairedDelimiter\floor{\lfloor}{\rfloor}

\allowdisplaybreaks

\begin{document}

%

%

\twocolumn[

\aistatstitle{Online Ranking with Top-1 Feedback}

\aistatsauthor{ Sougata Chaudhuri\And Ambuj Tewari}

\aistatsaddress{ University of Michigan, Ann Arbor \And University of Michigan, Ann Arbor } ]

\begin{abstract}
We consider a setting where a system learns to rank a fixed set of $m$ items. The goal is produce good item rankings for users with diverse interests who interact online with the system for $T$ rounds. We consider a novel top-$1$ feedback model: at the end of each round, the relevance score for only the top ranked object is revealed. However, the performance of the system is judged on the entire ranked list. We provide a comprehensive set of results regarding learnability under this challenging setting. For PairwiseLoss and DCG, two popular ranking measures, we prove that the minimax regret is $\Theta(T^{2/3})$. Moreover, the minimax regret is achievable using an efficient strategy that only spends $O(m \log m)$ time per round. The same efficient strategy achieves $O(T^{2/3})$ regret for Precision@$k$. Surprisingly, we show that for normalized versions of these ranking measures, i.e., AUC, NDCG \& MAP, no online ranking algorithm can have sublinear regret.
\end{abstract}

\section{Introduction}

Consider a system that is learning to rank a fixed set of objects for presentation to users, when different users have varied preferences for the objects. Learning occurs in an online setting: at each round, the system outputs a ranked list of the objects and the quality of ranking is measured by one of several popular ranking measures (like DCG \citep{jarvelin2000} or MAP \citep{baeza1999}), taking into account the users' preferences encoded as relevance vectors. We work in a game-theoretic setting and do not make any stochastic assumptions on how the relevance vectors are generated. Thus, it is assumed that the relevance vectors are generated by an oblivious, non-stochastic adversary. The objective of the learner is to have a sub-linear (in the number of rounds) regret against the best ranking in hindsight. The idea of ranking for diverse preferences has been motivated from a branch of work, sometimes called ``ranking with diversity".

Most of the existing work on``ranking with diversity"  \citep{radlinski2008,radlinski2009,agrawal2009} has focused on learning an optimal ranking of a fixed set of objects with a simple $0$-$1$ loss. The loss in a round is $0$  if among the top $k$ (out of $m$) objects presented to a user, the user finds at least one relevant object. Our model focuses on optimal ranking where the losses considered are practical ranking losses like DCG and MAP.  In addition, we consider a novel and challenging feedback model in this paper: \emph{the learner only gets to see the relevance of the object placed at the top} (rank $1$), whereas the ranking performance measure, and hence the regret, depends on the full relevance vector. Of course, one can consider a top-$c$ feedback model for a constant $c \ge 1$ that doesn't grow with $m$. We choose to focus on the most challenging $c=1$ case. We highlight two practical scenarios motivating the feedback model.

{\bf Economic Constraints}: A company wants to produce a ranked order of a fixed set of products related to a query. Different products are likely to have varying relevance to different users, depending on user characteristics such as age, gender, etc.  In principle, a user can browse through the entire ranked list giving carefully considered ratings, say on a 5 point scale, to each product. In practice, however, it is quite likely that user will scan through all the products and have a rough idea about how relevant each product is to her. But she will likely be reluctant to give thorough feedback on each product, unless the company provides some economic incentives to do so. Though the company needs high quality feedback on each product to keep refining the ranking strategy, it cannot afford to give incentives due to budget constraints. Hence, they require the user to give feedback only on top placed product. This allows the user to look at all products but does not burden her with task of providing feedback beyond the top-ranked product.  In this scenario, a full relevance vector is implicit in the user's mind but the system (company) gets to see, and possibly pays for, the relevance of only the top placed product.

{\bf User Burden Constraints}: A medical company wants to build an app to suggest activities (take a walk, meditate, watch relaxing videos, etc.) that can lead to reduction of stress in a certain highly stressed segment of the population. Not all activities are suitable for everyone under all conditions since the effects of the activities vary depending on the user attributes like age \& gender and on the context such as time of day \& day of week. To satisfy most users, the company wants to produce a useful ordering of the stress reduction activities, but stressed users are unlikely to give feedback on the usefulness (relevance) of every activity because it increases their cognitive burden. So the company can ask for feedback about just the top ranked activity while, as in the previous example, each activity has an implicit relevance score for each user. However, the system will only get to see the relevance of the top ranked suggested activity.

Theoretically, the top-1 feedback model is neither full-feedback nor bandit-feedback since not even the loss (quantified by some ranking measure) at each round is revealed to the learner. The appropriate framework to study the problem is that of \emph{partial monitoring} \citep{cesa2006}.  A very recent paper shows another practical application of the partial monitoring framework where the feedback is neither full nor bandit \citep{lincombinatorial}. Recent advances in the classification of partial monitoring games tell us that the minimax regret, in an adversarial setting, is governed by a property of the loss and feedback functions called \emph{observability} \citep{bartok2013, foster2011}.  Observability is of two kinds: \emph{local} and \emph{global}. We instantiate these general observability notions for the top-1 feedback case and prove that, for some ranking measures, namely PairwiseLoss \citep{duchi2010}, DCG and Precision@$k$ \citep{liu2007}, global observability holds. This immediately shows that the upper bound on regret scales as $O(T^{2/3})$. Specifically for PairwiseLoss and DCG, we further prove that local observability fails, which shows that their \emph{minimax} regret scales as $\Theta(T^{2/3})$. However, the generic algorithm that enjoys $O(T^{2/3})$ regret for globally observable games maintains an explicit distribution over learner actions. For us, the action set is the exponentially large set of $m!$ rankings over $m$ objects. We therefore provide an \emph{efficient} algorithm that exploits the structure of rankings. It runs in $O(m \log m)$ time per step and achieves a $O(T^{2/3})$ regret bound for PairwiseLoss, DCG and Precision@$k$. Moreover, the regret of our efficient algorithm has a logarithmic dependence on number of learner's actions (i.e., polynomial dependence on $m$), whereas the generic algorithm has a linear dependence on number of actions (i.e., exponential dependence on $m$).

For several measures, their \emph{normalized} versions are considered. For example, the normalized versions of PairwiseLoss, DCG and Precision@$k$ are called AUC \citep{cortes2004}, NDCG \citep{jarvelin2002} and MAP respectively. We show an unexpected result for the normalized versions: \emph{they do not} admit sub-linear regret algorithms under top-$1$ feedback. This is despite the fact that the opposite is true for their unnormalized counterparts! Intuitively, the normalization makes it hard to construct an unbiased estimator of the (unobserved) relevance vector. Surprisingly, we are able to translate this intuitive hurdle into a provable impossibility. Finally, we present some preliminary experiments to explore the performance of our efficient algorithm and compare its regret to its full information counterpart.


\section{Notation and Preliminaries}
\label{preliminaries}
We have a fixed set of $m$ objects numbered $\{1,2,\ldots,m\}$. A permutation $\sigma$ gives a mapping from objects to their ranks and its inverse $\sigma^{-1}$ gives a mapping from ranks back to objects. Thus, $\sigma(i)=j$ means object $i$ is placed at position $j$ while $\sigma^{-1}(i)=j$ means object $j$ is placed at position $i$. For a binary relevance vector $r \in \{0,1\}^m$, $r(i)$ indicates relevance level of object $i$. We denote $\{1,\ldots,n\}$ by $[n]$. The learner can choose from $m!$ actions (permutations) whereas nature/adversary can choose from $2^m$ outcomes (when relevance levels are restricted to binary) or from $n^m$ outcomes (when there are $n$ relevance levels). We sometimes refer to the $i$th player action (in some fixed ordering of $m!$ available actions) as $\sigma_i$ (resp. $i$th adversary action as $r_i$). With this convention, $\sigma(i)$ is a number but $\sigma_i$ is a permutation. Also, a vector can be row or column vector depending on context.

The oblivious adversary chooses the relevance vectors $r_t$ in advance but doesn't reveal them to the learner. At round $t$, the learner outputs a permutation (ranking) $\sigma_t$ of the objects (possibly using some internal randomization, based on feedback history so far). The quality of $\sigma_t$ is judged against $r_t$ by a ranking loss $RL$. \emph{Crucially, only the relevance of the top ranked object (i.e., $r_t(\sigma_t^{-1}(1))$) is revealed to the learner at end of round $t$}. Thus, the learner gets to know neither $r_t$ (full information problem) nor $RL(\sigma_t,r_t)$ (bandit problem). The objective of the learner is to minimize the expected regret with respect to best permutation in hindsight:
\begin{equation}
\begin{aligned}
\label{eq:objectiveregret}
\E_{\sigma_1, \ldots, \sigma_T}\left[ \sum_{t=1}^T RL(\sigma_t, r_t) \right] - \underset{\sigma}{\min} \ \sum_{t=1}^T RL(\sigma, r_t) .
\end{aligned}
\end{equation}
When $RL$ is a gain, not loss, we need to negate the quantity above.
We consider binary relevance but many of our techniques should easily extend to multi-graded relevance provided the performance measure has the right form.
The worst-case regret of a learner strategy is its maximal regret over all possible choices of $r_1,\ldots,r_T$. The {\bf minimax regret} is the minimal worst-case regret over all learner strategies.


\section{Ranking Measures}
\label{rankingmeasures}

We consider ranking measures which can be expressed in the form $f(\sigma) \cdot r$, where the function $f:\mathbb{R}^m \rightarrow \mathbb{R}^m$ is composed of $m$ copies of a univariate, monotonic, scalar valued function. Thus, $f(\sigma)= [f^s(\sigma(1)), f^s(\sigma(2)), \ldots, f^s(\sigma(m))]$, where $f^s: \mathbb{R} \rightarrow \mathbb{R}$. Monotonic (increasing) means $f^s(\sigma(i))\ge f^s(\sigma(j))$, whenever $\sigma(i) > \sigma(j)$. Monotonic (decreasing) is defined similarly. 
The following popular ranking measures can be expressed in the form $f(\sigma) \cdot r$.

{\bf PairwiseLoss \& SumLoss}: PairwiseLoss is defined as: $PL(\sigma,r)= \sum_{i=1}^m \sum_{j=1}^m \mathbbm{1}(\sigma(i) < \sigma(j))\mathbbm{1}(r(i) < r(j))$. PairwiseLoss cannot be directly expressed in the form of $f(\sigma)\cdot r$. Instead, we consider {\bf SumLoss}, defined as: $SumLoss(\sigma,r)= \sum_{i=1}^m \sigma(i)\ r(i)$. SumLoss has the form $f(\sigma) \cdot r$, where $f(\sigma)=\sigma$. 
It has been shown by \citet{ailon2014} that SumLoss differs from PairwiseLoss only by an $r$-dependent constant and hence the regret under the two measures are equal:
{\small
\begin{multline}
\label{rankloss-sumloss}
\sum_{t=1}^T PL(\sigma_t, r_t) -  \sum_{t=1}^T PL(\sigma, r_t) = \\
\sum_{t=1}^T SumLoss(\sigma_t, r_t) - \sum_{t=1}^T SumLoss(\sigma,r_t) .
\end{multline}
}

{\bf Discounted Cumulative Gain}: DCG, which admits non-binary relevance vectors, is defined as: 
$DCG(\sigma,r)= \sum_{i=1}^m \frac{2^{r(i)}-1}{\log_2(1+ \sigma(i))}$ and becomes $\sum_{i=1}^m \frac{r(i)}{\log_2(1+ \sigma(i))}$ for $r (i) \in \{0,1\}$. 
Thus, for binary relevance, $DCG(\sigma,r)$  has the form $f(\sigma) \cdot r$, where $f(\sigma)= [\frac{1}{\log_2(1+ \sigma(1))}, \frac{1}{\log_2(1+ \sigma(2))}, \ldots, \frac{1}{\log_2(1+ \sigma(m))}]$.

{\bf Precision@k Gain}: Precision@$k$ is defined as $Prec@k(\sigma,r)= \sum_{i=1}^m \mathbbm{1}(\sigma(i) \le k)\ r(i)$. Precision@$k$ can be written as $f(\sigma) \cdot r$ where $f(\sigma)=  [\mathbbm{1}(\sigma(1)<k), \ldots, \mathbbm{1}(\sigma(m)<k)]$. Our focus is on $k \ge 2$, since for $k=1$, top-$1$ feedback is actually the same as full information feedback, for which efficient algorithms exist. 

{\bf Normalized measures are not of the form ${\bf f(\sigma) \cdot r}$}: PairwiseLoss, DCG and Precion@$k$ are unnormalized versions of popular ranking measures, namely, Area Under Curve (AUC), Normalized Discounted Cumulative Gain (NDCG) and Mean Average Precision (MAP) respectively.  None of these can be expressed in the form $f(\sigma) \cdot r$.

{\bf NDCG}: $NDCG(\sigma,r)= \frac{1}{Z(r)}\sum_{i=1}^m \frac{2^{r(i)}-1}{\log_2(1+ \sigma(i))}$ and becomes $\frac{1}{Z(r)} \sum_{i=1}^m \frac{r(i)}{\log_2(1+ \sigma(i))}$ for $r(i) \in \{0,1\}$.  Here $Z(r)= \underset{\sigma}{\max} \sum_{i=1}^m  \frac{2^{r(i)}-1}{\log_2(1+ \sigma(i))}$ is the normalizing factor ($Z(r)= \underset{\sigma}{\max} \sum_{i=1}^m  \frac{r(i)}{\log_2(1+ \sigma(i))}$ for binary relevance). It can be clearly seen that $NDCG(\sigma, r)= f(\sigma) \cdot g(r)$, where $f(\sigma)$ is same as DCG but $g(r)= \frac{r}{Z(r)}$ is non-linear in $r$. 

{\bf MAP}: MAP is a gain function and is defined as: $MAP(\sigma,r)= \frac{1}{\|r\|_1}\sum\limits_{\substack{i=1}}^m \frac{\sum\limits_{j\le i} \mathbbm{1}(r(\sigma^{-1}(j))=1)}{i} \mathbbm{1}(r(\sigma^{-1}(i)=1)$. It can be clearly seen that MAP cannot be expressed in the form $f(\sigma) \cdot r$. 

{\bf AUC}: AUC is a loss function and is defined as: $AUC(\sigma,r)= \frac{1}{N(r)} \sum_{i=1}^m \sum_{j=1}^m \mathbbm{1}(\sigma(i) < \sigma(j))\mathbbm{1}(r(i) < r(j))$, where $N(r)=  (\sum_{i=1}^m \mathbbm{1}(r(i)=1)) \cdot (m-\sum_{i=1}^m \mathbbm{1}(r(i)=1))$. It can be clearly seen that AUC cannot be expressed in the form $f(\sigma) \cdot r$. 

All subsequent results will be for binary valued relevance vectors, unless stated otherwise.

\section{Summary of Results}
We summarize our main results here before delving into technical details. The regret bounds are over time horizon $T$, with learner playing against an \emph{oblivious adversary}. Unless otherwise stated, all proofs and extensions are given in the appendix.

{\bf Result 1}: The minimax regret under DCG and PairwiseLoss (and hence SumLoss) is $\Theta(T^{2/3})$.

{\bf Result 2}: An efficient algorithm, with running time $O(m \log m)$ per step, achieves the minimax regret under DCG and PairwiseLoss and also has a regret of $O(T^{2/3})$ for Precision@$k$. \emph{The precise minimax regret under Precision@$k$, $k\ge2$, remains an open issue}.

{\bf Result 3}: The minimax regret for any of the normalized versions -- NDCG, MAP and AUC -- is $\Theta(T)$. Thus, there is no algorithm that guarantees \emph{sublinear} regret for the normalized measures.

{\bf Result 4}: The minimax regret rate, as a function of $T$, both for DCG and NDCG, does not change (i.e., remains $\Theta(T^{2/3})$ and $\Theta(T)$ respectively) when we consider non-binary, multi-graded relevance vectors.

\section{Relevant Definitions from Partial Monitoring}
\label{partialmonitoring}
We develop all results in context of SumLoss. We then extend the results to other ranking measures. 
Our main results on regret bounds build on some of the theory for abstract partial monitoring games developed by \citet{bartok2013} and \citet{foster2011}.  For ease of understanding, we reproduce the relevant notations and definitions in context of SumLoss. 

{\bf Loss and Feedback Matrices}: The online learning game with the SumLoss measure and feedback being relevance of top ranked object, can be expressed in form of a pair of \emph{loss matrix} and \emph{feedback matrix}. The \emph{loss matrix} $L$ is an $m! \times 2^m$ dimensional matrix, with rows indicating the learner's actions (permutations) and columns representing adversary's actions (relevance vectors). The entry in cell $(i,j)$ of $L$ indicates loss suffered when learner plays action $i$ (i.e., $\sigma_i$) and adversary plays action $j$ (i.e., $r_j$), that is, $L_{i,j}= \sigma_i \cdot r_j= \sum_{k=1}^m \sigma_i(k)r_j(k)$. The \emph{feedback matrix} $H$ has same dimension as \emph{loss matrix}, with $(i,j)$ entry being the relevance of top ranked object, i.e., $H_{i,j}= r_j(\sigma_i^{-1}(1))$. When the learner plays action $\sigma_i$ and adversary plays action $r_j$, the true loss is $L_{i,j}$, while the feedback received is $H_{i,j}$. 

Table \ref{lossmatrix-table} and \ref{feedbackmatrix-table} illustrate the matrices, with number of objects $m=3$. In both the tables, the permutations indicate rank of each object and relevance vector indicates relevance of each object. For example, $\sigma_5= 3 1 2$ means object $1$ is ranked $3$, object $2$ is ranked $1$ and object $3$ is ranked 2. $r_5=100$ means object $1$ has relevance level $1$ and other two objects have relevance level $0$. Also, $L_{3,4}= \sigma_3 \cdot r_4= \sum_{i=1}^3 \sigma_3(i) r_4(i)= 2\cdot 0 + 1 \cdot 1 + 3 \cdot 1= 4$;  $ H_{3,4}= r_4(\sigma_3^{-1}(1))= r_4(2)= 1$. Other entries are computed similarly.

\begin{table}[t] 
\caption{Loss matrix $L$ for $m=3$} 
\label{lossmatrix-table}
\begin{center}
\tabcolsep=0.110cm
\begin{tabular}{c c c c c c c c c}  
\hline 
Objects & $r_1$ & $r_2$ & $r_3$ & $r_4$ & $r_5$ & $r_6$ & $r_7$ & $r_8$ \\ [0.4ex] 
\hline 123 & 000 & 001 & 010 & 011 & 100 & 101 & 110 & 111\\ [0.4ex]
\hline 
$\sigma_1= 1 2 3$ & 0 & 3 & 2 & 5 & 1 & 4 & 3 & 6 \\ 
$\sigma_2= 1 3 2$ & 0 & 2 & 3 & 5 & 1 & 3 & 4 & 6 \\ 
$\sigma_3= 2 1 3$ & 0 & 3 & 1 & 4 & 2 & 5 & 3 & 6  \\ 
$\sigma_4= 2 3 1$ & 0 & 1 & 3 & 4 & 2 & 3 & 5 & 6 \\ 
$\sigma_5= 3 1 2$ & 0 & 2 & 1 & 3 & 3 & 5 & 4 & 6 \\
$\sigma_6= 3 2 1$ & 0 & 1 & 2 & 3 & 3 & 4 & 5 & 6 \\
\hline 
\end{tabular} 
\end{center}
\end{table}


\begin{table}[t] 
\caption{Feedback matrix $H$ for $m=3$} 
\label{feedbackmatrix-table}
\begin{center}
\tabcolsep=0.110cm
\begin{tabular}{c c c c c c c c c}  
\hline 
Objects & $r_1$ & $r_2$ & $r_3$ & $r_4$ & $r_5$ & $r_6$ & $r_7$ & $r_8$ \\ [0.4ex] 
\hline 123 & 000 & 001 & 010 & 011 & 100 & 101 & 110 & 111\\ [0.4ex]
\hline 
$\sigma_1= 1 2 3$ & 0 & 0 & 0 & 0 & 1 & 1 & 1 & 1 \\ 
$\sigma_2= 1 3 2$ & 0 & 0 & 0 & 0 & 1 & 1 & 1 & 1 \\ 
$\sigma_3= 2 1 3$ & 0 & 0 & 1 & 1 & 0 & 0 & 1 & 1  \\ 
$\sigma_4= 2 3 1$ & 0 & 1 & 0 & 1 & 0 & 1 & 0 & 1 \\ 
$\sigma_5= 3 1 2$ & 0 & 0 & 1 & 1 & 0 & 0 & 1 & 1 \\
$\sigma_6= 3 2 1$ & 0 & 1 & 0 & 1 & 0 & 1 & 0 & 1 \\
\hline 
\end{tabular} 
\end{center}
\end{table}

Let $\ell_i \in \mathbb{R}^{2^m}$ denote row $i$ of $L$. Let $\Delta$ be the probability simplex in $\mathbb{R}^{2^m}$, i.e., $\Delta= \{ p \in \mathbb{R}^{2^m}: \forall \ 1 \le i \le 2^m, \ p_i \ge 0, \ \sum p_i = 1\}$.  
The following definitions, given for abstract problems by \cite{bartok2013}, has been refined to fit our problem context.

{\bf Definition 1}: Learner action $i$ is called optimal under distribution $p \in \Delta$, if $\ell_i \cdot p \le \ell_ j \cdot p$, for all other learner actions $1 \le j \le m!, \ j \neq i$.  For every action $i \in [m!]$, probability cell of $i$ is defined as $C_i =\{ p \in \Delta: \text{action } i \text{ is optimal under } p\}$. If a non-empty cell $C_i$ is $2^m-1$ dimensional (i.e, elements in $C_i$ are defined by only 1 equality constraint), then associated action $i$ is called \emph{Pareto-optimal}.

Note that since entries in $H$ are relevance levels of objects, there can be maximum of $2$ distinct elements in each row of $H$, i.e., $0$ or $1$ (assuming binary relevance). 

{\bf Definition 2}: The \emph{signal matrix} $S_i$, associated with learner's action $\sigma_i$, is a matrix with 2 rows and $2^m$ columns, with each entry $0$ or $1$, i.e., $S_i \in \{0,1\}^{2 \times 2^m}$. The entries of $\ell$th column of row $1$ and $2$ of $S_i$ are respectively: $(S_i)_{1,\ell}= \mathbbm{1}(H_{i,\ell}=0)$ and $(S_i)_{2,\ell}= \mathbbm{1}(H_{i,\ell}=1)$. 

Note that by definitions of signal and feedback matrices, the 2nd row of $S_i$ (2nd column of $S^{\top}_i)$) is precisely the $i$th row of $H$. The 1st row of $S_i$ (1st column of $S^{\top}_i)$) is the (boolean) complement of $i$th row of $H$.

\section{Minimax Regret for SumLoss}
The minimax regret for SumLoss will be established by showing that: a) SumLoss satisfies \emph{global observability}, and b) it does not satisfy \emph{local observability}. 

\subsection{Global Observability}
\label{global}
 
{\bf Definition 3}: The condition of \emph{global observability} holds, w.r.t. loss matrix $L$ and feedback matrix $H$, if for every pair of learner's actions $\{\sigma_i,\sigma_j\}$, it is true that $\ell_i -\ell_j \in \oplus_{k \in [m!]} Col(S_k^{\top})$ (where $Col$ refers to column space). 

The global observability condition states that the (vector) loss difference between any pair of learner's actions has to belong to the vector space spanned by columns of (transposed) signal matrices corresponding to all possible learner's actions.
We derive the following theorem on global observability for $SumLoss$.

\begin{thm}
\label{globalobservability}
The global observability condition, as per Definition 3, holds w.r.t. loss matrix  $L$ and feedback matrix $H$ defined for SumLoss, for any $m \ge 1$.
\end{thm}

\begin{proof}
For any $\sigma_a$ (learner's action) and $r_b$ (adversary's action), we have 
\begin{equation*}
\begin{aligned}
L_{a,b}=\sigma_a \cdot r_b= & \sum_{i=1}^m \sigma_a(i)r_b(i) \stackrel{1}{=} \sum_{j=1}^ m j\  r_b(\sigma_a^{-1}(j)) \stackrel{2}{=}  \\
& \sum_{j=1}^m j\ r_b(\tilde{\sigma}_{j(a)} ^{-1}(1)) \stackrel{3}{=} \sum_{j=1}^m j \ (S^{\top}_{\tilde{\sigma}_{j(a)}})_{r_b,2}.\\
\end{aligned}
\end{equation*}
Thus, we have
\begin{equation*}
\begin{aligned}
& \ell_a = \\
& [L_{a,1}, L_{a,2},\ldots, L_{a,2^m}]  = [L_{\sigma_a, r_1}, L_{\sigma_a,r_2},\ldots, L_{\sigma_a,r_{2^m}}] = \\
&  [\sum_{j=1}^m j \ (S^{\top}_{\tilde{\sigma}_{j(a)}})_{r_1,2}, \sum_{j=1}^m j \ (S^{\top}_{\tilde{\sigma}_{j(a)}})_{r_2,2}, .., \sum_{j=1}^m j \ (S^{\top}_{\tilde{\sigma}_{j(a)}})_{r_{2^m},2}]\\
& \stackrel{4}{=} \sum_{j=1}^m j \ (S^{\top}_{\tilde{\sigma}_{j(a)}})_{:,2} .
\end{aligned}
\end{equation*}

Equality 4 shows that $\ell_a$ is in the column span of $m$ of the $m!$ possible (transposed) signal matrices, specifically in the span of the 2nd columns of those (transposed) $m$ matrices. 
Hence, for all actions $\sigma_a$, it is holds that $ \ell_a \in { \oplus}_{k \in [m!]} Col(S_k^{\top})$. This implies that $\ell_a - \ell_b \in \oplus_{k \in [m!]} Col(S_k^{\top}), \ \forall \ \sigma_a,\sigma_b. $

1. Equality 1 holds because $\sigma_a(i)=j \Rightarrow i= \sigma_a^{-1}(j)$. 

2. Equality 2 holds because of the following reason. For any permutation $\sigma_a$ and for every $j \in [m]$, $\exists$ a permutation $ \tilde{\sigma}_{j(a)}$, s.t. the object which is assigned rank $j$ by $\sigma_a$ is the same object assigned rank $1$ by $\tilde{\sigma}_{j(a)}$, i.e., $\sigma_a^{-1}(j)= \tilde{\sigma}_{j(a)} ^{-1}(1)$.

3. In Equality 3, $(S^{\top}_{\tilde{\sigma}_{j(a)}})_{r_b,2}$ indicates the $r_b$th row and 2nd column of (transposed) signal matrix $S_{\tilde{\sigma}_{j(a)}}$, corresponding to learner action $\tilde{\sigma}_{j(a)}$. Equality 3 holds because $r_b(\tilde{\sigma}_{j(a)} ^{-1}(1))$ is the entry in the row corresponding to action $\tilde{\sigma}_{j(a)}$ and column corresponding to action $r_b$ of $H$ (see Definition 2). 

4. Equality 4 holds from the observation that for a particular $j$, $[(S^{\top}_{\tilde{\sigma}_{j(a)}})_{r_1,2}, (S^{\top}_{\tilde{\sigma}_{j(a)}})_{r_2,2}, \ldots, (S^{\top}_{\tilde{\sigma}_{j(a)}})_{r_{2^m},2}]$ forms the 2nd column of  $(S^{\top}_{\tilde{\sigma}_{j(a)}})$, i.e.,  $(S^{\top}_{\tilde{\sigma}_{j(a)}})_{:,2}$.
\end{proof}

\subsection{Local Observability}
\label{local}

{\bf Definition 4}: Two Pareto-optimal (learner's) actions $i$ and $j$ are called \emph{neighboring actions} if $C_i \cap C_j$ is a $(2^m-2)$ dimensional polytope (where $C_i$ is probability cell of action $\sigma_i$). The \emph{neighborhood action set} of two neighboring (learner's) actions $i$ and $j$ is defined as $N^{+}_{i,j}= \{ k \in [m!]: C_i \cap C_j \subseteq C_k\}$. 

{\bf Definition 5}: A pair of neighboring (learner's) actions $i$ and $j$ is said to be locally observable if  $\ell_i -\ell_j \in \oplus_{k \in N^{+}_{i,j}} Col(S_k^{\top})$. The condition of \emph{local observability} holds if every pair of neighboring (learner's) actions is locally observable. 

We now show that local observability condition fails for $L, H$ under SumLoss. First, we present the following two lemmas characterizing Pareto-optimal actions and neighboring actions for SumLoss.

\begin{lem}
\label{pareto-optimal}
For SumLoss, each learner's action $i$ is Pareto-optimal, where Pareto-optimality has been defined in Definition 1.
\end{lem}

\begin{lem}
\label{neighbor-actions}
A pair of learner's actions $\{\sigma_i, \ \sigma_j\}$ is a neighboring actions pair, if there is exactly one pair of objects, numbered \{a, b\}, whose positions differ in $\sigma_i$ and $\sigma_j$. Moreover, $a$ needs to be placed just before $b$ in $\sigma_i$ and $b$ needs to placed just before $a$ in $\sigma_j$. 
\end{lem}

Lemma \ref{pareto-optimal} and \ref{neighbor-actions} lead to following result.
\begin{thm}
\label{localobservability}
The local observability condition, as per Definition 5, fails w.r.t. loss matrix $L$ and feedback matrix $H$ defined for SumLoss, already at $m=3$.
\end{thm}

\subsection{Minimax Regret Bound}

We establish the minimax regret for SumLoss by combining results on global and local observability. First, we get a lower bound by combining our Theorem~\ref{localobservability} with Theorem 4 of \cite{bartok2013}.
\begin{cor}
\label{sumlossminimaxregret}
Consider the online game for SumLoss with top-$1$ feedback and $m=3$. Then, for every online learning algorithm, there is an adversary strategy generating relevance vectors, that guarantees the following
{\small
\begin{equation}
\begin{aligned}
\label{minimax-regret-sumloss}
\E\left[\sum_{t=1}^T SumLoss(\sigma_t,r_t)\right] - & \underset{\sigma}{\min} \sum_{t=1}^T SumLoss(\sigma, r_t)  \\
= & \ \Omega(T^{2/3}) .
\end{aligned}
\end{equation}
}

where the expectation is taken w.r.t. randomized learner's actions.
\end{cor}

An immediate corollary of Theorem~\ref{globalobservability} and Theorem 3.1 in \cite{cesa2006} gives an in-efficient algorithm (inspired by the algorithm originally given in \cite{piccolboni2001}) obtaining $O(T^{2/3})$ regret. 

\begin{cor}
The algorithm in Figure 1 of \cite{cesa2006} achieves $O(T^{2/3})$ regret bound for SumLoss.
\end{cor}

The results above establish that the minimax regret for SumLoss, under top-$1$ feedback model, is $\Theta(T^{2/3})$. However, the algorithm in \cite{cesa2006} is intractable in our setting since the number of learner's actions is exponential in number of objects $m$. The next section tackles the efficiency issue.

\section{Efficient Algorithm for Obtaining Minimax Regret under SumLoss}
\label{efficientalgorithm}
We provide an efficient algorithm for getting an $O(\mathrm{poly}(m)T^{2/3})$ regret bound for SumLoss. The per round running time of the algorithm is $O(m \log m)$.


The key idea that we use in our algorithm is to divide time horizon $T$ into phases. Within each phase, we allot a small number of rounds for pure \emph{exploration} (this lets us estimate the average relevance vector for that phase). The estimated average vector is fed into a full information algorithm to get the distribution over actions for the next phase. Rounds in the next phase choose actions according to the distribution suggested by Follow the Perturbed Leader (FTPL) \cite{kalai2005} (this is \emph{exploitation} of previous experience). One of the key reasons for using FTPL as the full information algorithm, instead of exponential weighing schemes, is that the structure of our problem allows the FTPL update to be implemented via a simple sorting operation on $m$ objects. Exponential weighting schemes would explicitly maintain distribution over $m!$ actions, a prohibitively expensive step.

Our algorithm is motivated by the reduction from bandit-feedback to full feedback given by \cite{nisan2007}. However, the reduction \emph{cannot be directly applied to our problem}, because we are not in the bandit setting and hence do not know loss of any action. Further, the algorithm of \cite{nisan2007} spends $N$ rounds per phase to try out \emph{each} of the $N$ available actions --- this is infeasible in our setting since $N = m!$. 

\paragraph{Discussion of Algorithm~\ref{alg:top-1}.} Our algorithm RTop-1F divides the time horizon into equal sized blocks of size $K$ (lines 2-3). At the beginning of each block, $m$ time points are selected uniformly at random  without replacement in that block (lines 8-9). Within each block, if the current time is one of the selected times for exploration, an arbitrary permutation that places a particular object on top is played (lines 12-14). Otherwise, the permutation which minimizes the dot product with the perturbed score vector is played (lines 16-19). Note that the step $\sigma_t = M(\hat{s}_i + p_t)$ requires sorting of the $m$ objects, which takes $O(m \log{m})$ time. Our main theorem on regret of Algorithm~\ref{alg:top-1} is as follows.
\begin{thm}
\label{efficientregret}
The expected regret of SumLoss, obtained by applying Algorithm \ref{alg:top-1}, with $K= m^{-1/3}T^{2/3}$ and $\epsilon= \sqrt{\frac{1}{mK}}$, and the expectation being taken over random learner's actions $\sigma_t$, is
\begin{multline}
\E\left[\sum_{t=1}^T SumLoss(\sigma_t, r_t)\right] \le \\
 \underset{\sigma}{\min}\sum_{t=1}^T SumLoss(\sigma_t, r_t) + O(m^{8/3}T^{2/3}) .
\end{multline}
\end{thm}


\floatstyle{ruled}
\newfloat{algorithm}{htbp}{loa}
\floatname{algorithm}{Algorithm}
\begin{algorithm}
{\small
\caption{RankingwithTop-1Feedback(RTop-1F)}
\label{alg:top-1}
\begin{tabbing}
tabs \= tabs \= tabs \= tabs \kill
1: $T=$ Time horizon, $K=$ No. of (equal sized) blocks, \\
2: Time horizon divided into blocks $\{B_1,\ldots,B_K\}$,\\ 
3: where,  $B_i= \{(i-1)(T/K)+1, \ldots, i (T/K)\}$. \\
4: Randomization parameter $\epsilon$.\\
5: Initialize $\hat{s}_0=\mathbf{0} \in \mathbb{R}^m$, $\hat{r}_0=\mathbf{0} \in \mathbb{R}^m$.\\
6: {\bf For} \= $i= 1,\ldots,K$\\
7:\> Update $\hat{s}_i= \hat{s}_{i-1} + \hat{r}_{i-1}$.\\
8:\> Select $m$ time points $\{i_1,\ldots,i_m\}$ from block $B_i$,\\
9:\> uniformly at random, without replacement.\\
10:\> {\bf For} \=  $t \in B_i$\\
11:\>\> {\bf If} $t = i_j \in \{i_1,\ldots,i_m\}$\\
12:\>\>\> Output any permutation $\sigma_t$ which places\\ 
13:\>\>\> $j$th object on top.\\
14:\>\>\> Receive feedback on the $j$th object  $r_{i_j}(j)$.\\
15:\>\> {\bf Else} \\
16:\>\>\> Sample $p_t \in [0,1/\epsilon]^m$ from the product \\
17:\>\>\> of uniform distribution in each dimension.\\
18:\>\>\> Output permutation $\sigma_{t}= M(\hat{s}_i + p_t)$ \\
19:\>\>\> where $M(y)= \underset{\sigma}{\argmin} \ \sigma\cdot y$. \\
20:\> {\bf end for}\\
21:\> Set $\hat{r}_i= [r_{i_1}(1),\ldots,r_{i_m}(m)] \in \mathbb{R}^m$.\\
22:{\bf end for}
\end{tabbing}
}
\end{algorithm}

The following simple but useful lemma is required to prove Theorem \ref{efficientregret}.

\begin{lem}
\label{unbiasedestimator}
Let the average of full relevance vectors over the time period $\{1,2,\ldots,t\}$ be denoted as $r_{1:t}^{avg}$, that is, $r_{1:t}^{avg}= \sum_{k=1}^{t}\dfrac{r_k}{t}$. Let $\{i_1,i_2,\ldots,i_m\}$ be $m$ arbitrary time points, chosen uniformly at random, without replacement, from $\{1,\ldots,t\}$. At time point $i_j$, only the $j$th component of vector $r_{i_j}$, i.e., $r_{i_j}(j)$, becomes known, $\forall j \in \{1,\ldots,m\}$. Then the relevance vector $\hat{r}_{t}= [r_{i_1}(1),\ldots,r_{i_m}(m)]$ is an unbiased estimator of $r_{1:t}^{avg}$.
\end{lem}

\section{Regret Bounds for PairwiseLoss, DCG and Prec@k}
As we saw in Eq.~\ref{rankloss-sumloss}, the regret of SumLoss is same as regret of PairwiseLoss. Thus, SumLoss in Cor.~\ref{sumlossminimaxregret} and Thm.~\ref{efficientregret} can be replaced by PairwiseLoss to get exactly same results on regret.

All the results of SumLoss can be extended to DCG. Moreover, the results can be extended even for discrete, non-binary relevance vectors. Thus, the minimax regret of DCG, when the adversary can take any discrete valued, non-negative relevance vector is $\Theta(T^{2/3})$, which can be achieved by (a slight variant of) the efficient algorithm of Sec.~\ref{efficientalgorithm}. The main differences between SumLoss and DCG are the following. The former is a loss function, the latter is a gain function. Also, $f(\sigma) \neq \sigma$  in DCG (Def. in Sec.\ref{preliminaries} ) and when $r \in \{0,1,\ldots,n\}^m$, DCG cannot be expressed as $f(\sigma) \cdot r$, as is clear from definition in Sec.~\ref{rankingmeasures}. Nevertheless, DCG can be expressed as $f(\sigma) \cdot g(r)$, , where $g(r)= [g^s(r(1)), g^s(r(2)), \ldots, g^s(r(m))], \ g^s(i)= 2^i -1$ is constructed from univariate, monotonic, scalar valued functions. Thus, there are minor differences in definitions and proofs of theorems for SumLoss and DCG. The structural properties of $f(\cdot)$, $g(\cdot)$ are key in extending results. For binary valued relevance vectors, Algorithm~\ref{alg:top-1} can be applied to DCG as is. For multi-graded relevance vector, the only thing that changes is  that the relevance feedback is transformed via component functions of $g(\cdot)$. 

We provide the extension of Theorem~\ref{efficientregret} for DCG. Let relevance vectors chosen by adversary be of level $n+1$, i.e., $r \in \{0,1,\ldots, n\}^m$. In practice, $n$ is almost always less than $5$.
\begin{thm}
\label{efficientregretforDCG}
The expected regret of $DCG$, obtained by applying Algorithm \ref{alg:top-1} , with $K= m^{-1/3}T^{2/3}$ and $\epsilon= \sqrt{\frac{1}{(2^n-1)^2mK}}$, and the expectation being taken over random learner's actions $\sigma_t$, is
\begin{multline}
\E\left[\sum_{t=1}^T DCG(\sigma_t, r_t)\right] \ge \underset{\sigma}{\max}\sum_{t=1}^T DCG(\sigma, r_t) \\
- O((2^n-1) m^{5/3}T^{2/3}) .
\end{multline}
In case of binary relevance vector, the regret term is $O(m^{5/3}T^{2/3})$. Moreover, since local observability fails, there is a matching $\Omega(T^{2/3})$ lower bound.
\end{thm}

The regret upper bounds we proved for SumLoss also easily extend to Precision@$k$. We have the following extension of Theorem~\ref{efficientregret}.

\begin{thm}
\label{efficientregretforPrec@k}
The expected regret of $Prec@k$, obtained by applying algorithm \ref{alg:top-1}, with $K= m^{-1/3}T^{2/3}$ and $\epsilon= \sqrt{\frac{1}{mK}}$, and the expectation being taken over random learner's actions $\sigma_t$, is
\begin{multline}
\E\left[\sum_{t=1}^T Prec@k(\sigma_t, r_t)\right]  \ge \underset{\sigma}{\max}\sum_{t=1}^T Prec@k(\sigma, r_t) \\
- O(k m^{2/3} T^{2/3}) .
\end{multline}
\end{thm} 

However, the value of $Prec@k$ is independent of the order of objects in the top $k$ positions of the ranked list. This changes the neighboring action claims. Therefore, the minimax regret of $Prec@k$ remains an open question, since
we do not have local observability failure results for $Prec@k$.

\section{Non-Existence of Sublinear Regret Bounds for NDCG, MAP and AUC}
As stated in Sec.~\ref{rankingmeasures}, NDCG, MAP and AUC are normalized versions of measures DCG, Precision@$k$ and PairwiseLoss. We have the following lemma for all these normalized ranking measures.

\begin{lem}\label{globalfailsfornormalized}
The global observability condition, as per Definition 1, fails for NDCG, MAP and AUC.
\end{lem}

Combining the above lemma with Theorem 2 of \cite{bartok2013}, we conclude that there \emph{cannot exist any algorithm which has sublinear regret for any of the following measures: NDCG, MAP or AUC, with top-1 feedback}.
\begin{thm}
There exists an online game for NDCG with top-1 feedback, such that for every online algorithm, there is an adversary strategy that guarantees the following
\begin{multline}
\label{minimax-regret}
\underset{\sigma}{\max} \sum_{t=1}^T NDCG(\sigma, r_t) - \E\left[ \sum_{t=1}^T NDCG(\sigma_t,r_t) \right]  \\
= \Omega (T) .
\end{multline}
Furthermore, the same lower bound holds if NDCG is replaced by MAP or AUC.
\end{thm}

{\bf Note}: In the NDCG case, allowing the adversary to play multigraded, and not just binary, relevance vectors only makes the adversary more powerful. So the lower bound above continues to apply.

\section{Simulation Results}

We conducted a simulation study to compare regret rate under the popular DCG metric when feedback is received only for top ranked object (by applying Algorithm~\ref{alg:top-1} ) with the case when full relevance vector is revealed at end of each round (by applying Follow the Perturbed Leader of \cite{kalai2005}). Relevance vectors were restricted to take binary values. The reason for choosing DCG is that it is a popular metric used in industry and to empirically confirm that Algorithm\ref{alg:top-1} works for DCG, even though the derivations focused on SumLoss. 

We simulated relevance vectors for a fixed set of 10 objects ($m=10$). We initially fixed half of the objects to be relevant and other half irrelevant, as the true relevance vector.  Then, binary valued relevance vectors for adversary were simulated by adding small Gaussian noise to the true relevance vector. Thus, there was mostly small variation among the relevance vectors, simulating the case that, in real world, majority of users might agree on the relevance of most objects, with small differences. A total of $T=10000$ relevance vectors were generated (simulating number of rounds). 

In Algorithm~\ref{alg:top-1}, since the average of the relevance vectors per block was estimated by uniform sampling according to Lemma~\ref{unbiasedestimator}, the algorithm was run $10$ times, with the same set of relevance vectors, for averaging under the algorithm's randomization. Fig.~\ref{Fig1} shows time-normalized regret with top-1 feedback for DCG. Time-normalized means the cumulative regret upto time $t$ was divided by $t$, for $1 \le t \le T$.  The figure clearly indicates that after the learning phase of the initial few iterations, the learner outputs mostly correct rankings, with the average regret going down to $0$ at rate $O(T^{-1/3})$. 

Fig.~\ref{Fig2} compares time-normalized regret, between top-1 and full information feedback, for DCG. The comparison was done from 1000 iterations onwards, i.e., roughly after the learning phase of the learner. It can be clearly seen that average regret with full information goes down at rate faster ($\Theta(T^{-1/2}) $) than average regret with top-1 feedback ($\Theta(T^{-1/3})$). 

%

\begin{figure}[h]
\begin{center}
\centerline{\includegraphics[height=60mm, width=95mm]{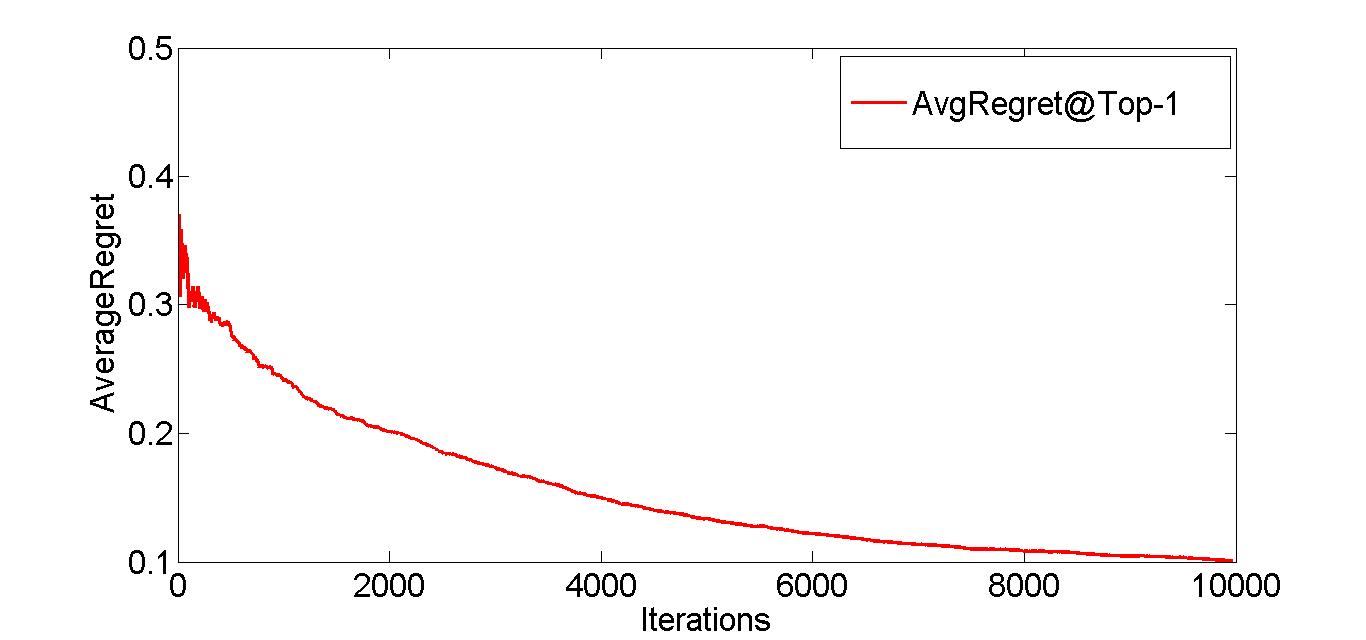}}
\caption{Average regret for DCG with feedback on top ranked object. \emph{Best viewed in color}. } \label{Fig1}
\end{center}
\end{figure} 

\begin{figure}[h]
\begin{center}
\centerline{\includegraphics[height=60mm, width=95mm]{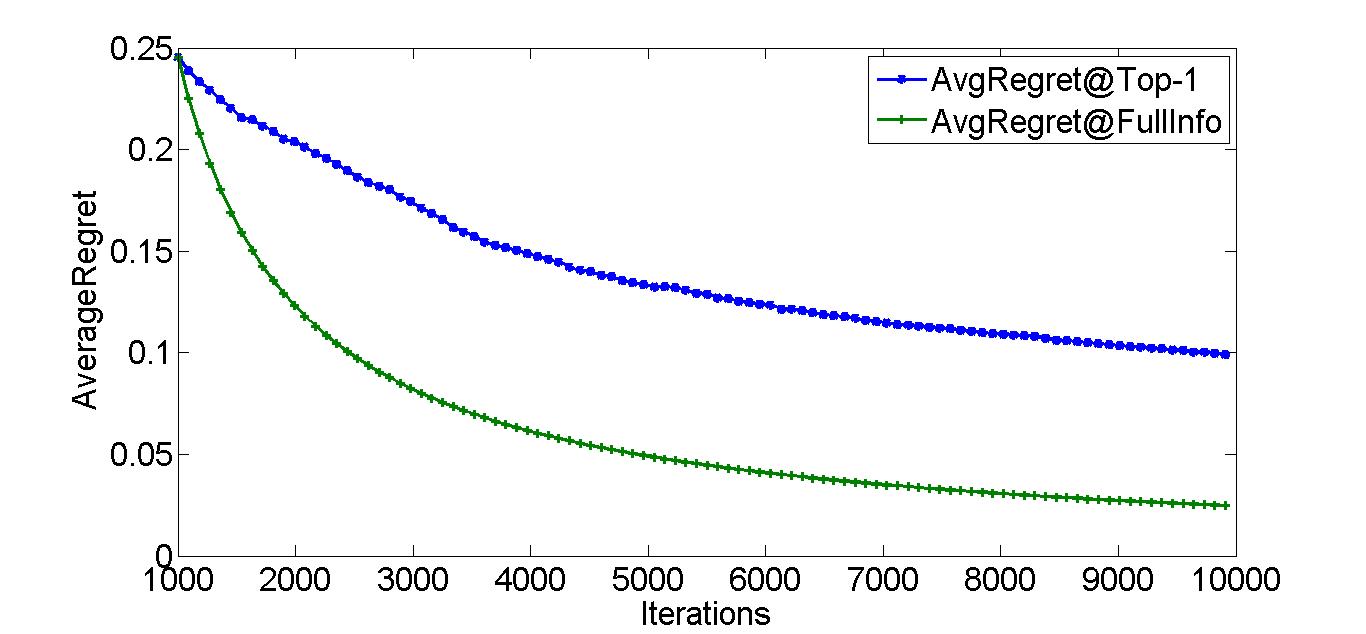}}
\caption{Comparison of average regret over time, for DCG, between top-1 feedback and full relevance vector feedback.  \emph{Best viewed in color. }.}\label{Fig4}
\label{Fig2}
\end{center}
\end{figure} 

\section{Conclusion}
We introduced a novel, interesting feedback model for online ranking of a fixed set of objects for users with diverse preferences. Our results are quite comprehensive as far as the $T$ dependence is concerned. The only exception is Precision@$k$ where the possibility of an $O(T^{1/2})$ regret algorithm remains open. Note that Precision@$k$ is really peculiar since top-$1$ feedback is actually full feedback when $k=1$.

The most interesting future extension of this work is to move beyond ranking of a fixed set of objects and considering different document lists associated with queries. This falls under the category of partial monitoring with side information. Very little relevant work has been done in the general setting and our current work can lay the foundations for interesting application in this field. Another extension is investigating whether an algorithm with sublinear regret can be defined for NDCG, MAP or AUC, when the regret is defined relative to some constant factor (larger than $1$) times the best performance in hindsight.

\section*{Acknowledgement}
The authors acknowledge the support of NSF under grant IIS-1319810.
\newpage

\bibliography{Diversity}
\bibliographystyle{plainnat}

\newpage

We provide missing proofs of theorems and extensions that were excluded from the main body of the paper due to space constraints.

\section{Regret for SumLoss}

\subsection{Proof of Lemma~\ref{pareto-optimal}}
\begin{proof}
For any $p \in \Delta$, we have $\ell_i \cdot p= \sum_{j=1}^{2^m}\  p_j \ (\sigma_i \cdot r_j) = \sigma_i \cdot (\sum_{j=1}^{2^m} p_j r_j)= \sigma_i \cdot E_{r}[r]$, where the expectation is taken w.r.t. $p$ ($p_j$ is the $j$th component of $p$). By dot product rule between 2 vectors, $l_i \cdot p$ is minimized when ranking of objects according to $\sigma_i$ and expected relevance of objects are in opposite order. That is, the object with highest expected relevance is ranked $1$ and so on. Formally, $l_i \cdot p$ is minimized when $E_{r}[r(\sigma_i^{-1}(1)] \ge E_{r}[r(\sigma_i^{-1}(2)] \ge \ldots \ge E_{r}[r(\sigma_i^{-1}(m)]$. 

Thus, for action $i$, probability cell is defined as $C_i= \{p \in \Delta: \sum_{j=1}^{2^m} p_j =1, \ E_{r}[r(\sigma_i^{-1}(1)] \ge E_{r}[r(\sigma_i^{-1}(2)] \ge \ldots \ge E_{r}[r(\sigma_i^{-1}(m)]\}$. Note that, $p \in C_i$ iff action $i$ is optimal w.r.t. $p$.  Since $C_i$ is obviously non-empty and it has only 1 equality constraint (hence $2^m -1$ dimensional), action $i$ is Pareto optimal. 

The above holds true for all learner's actions $i$. 
\end{proof}

\subsection{Proof of Lemma~\ref{neighbor-actions}}

\begin{proof}
From Lemma \ref{pareto-optimal}, we know that every one of learner's actions is Pareto-optimal and $C_i$, associated with action $\sigma_i$, has structure $C_i= \{p \in \Delta: \sum_{j=1}^{2^m} p_j =1, \ E_{r}[r(\sigma_i^{-1}(1)] \ge E_{r}[r(\sigma_i^{-1}(2)] \ge \ldots > E_{r}[r(\sigma_i^{-1}(m)]\}$.

Let $\sigma_i^{-1}(k)= a, \ \sigma_i^{-1}(k+1)=b$. Let it also be true that $\sigma_j^{-1}(k)= b, \ \sigma_j^{-1}(k+1)=a$ and $\sigma_i^{-1}(n)= \sigma_j^{-1}(n), \ \forall n \neq \{k, \ k+1\}$. Thus, objects in $\{\sigma_i,\sigma_j\}$ are same in all places except in a pair of consecutive places where the objects are interchanged.

Then, $C_i \cap C_j= \{p \in \Delta: \sum_{j=1}^{2^m} p_j =1, \ E_{r}[r(\sigma_i^{-1}(1)] \ge \ldots \ge E_{r}[r(\sigma_i^{-1}(k)] =  E_{r}[r(\sigma_i^{-1}(k+1)] \ge \ldots \ge E_{r}[r(\sigma_i^{-1}(m)]\}$. Hence, there are two equalities in the non-empty set $C_i \cap C_j$ and it is an $(2^m -2)$ dimensional polytope. Hence condition of Definition 4 holds true and $\{\sigma_i,\sigma_j\}$ are neighboring actions pair.

\end{proof}

\subsection{Proof of Theorem~\ref{localobservability}}

\begin{proof}
We will explicitly show that local observability condition fails by considering the case when number of objects is $m=3$. Specifically, action pair \{$\sigma_1, \ \sigma_2$\}, in Table \ref{lossmatrix-table} are neighboring actions, using Lemma \ref{neighbor-actions} . Now every other action $\{\sigma_3,\sigma_4,\sigma_5,\sigma_6\}$ either places object 2 at top or object 3 at top. It is obvious that the set of probabilities for which $E_r [r(1)]\ge E_r[r(2)]=E_r[r(3)]$ cannot be a subset of any $C_3,C_4,C_5, C_6$. From Def. 4, the neighborhood action set of actions $\{\sigma_1, \sigma_2\}$ is precisely $\sigma_1$ and $\sigma_2$ and contains no other actions.
By definition of signal matrices $S_{\sigma_1}, \ S_{\sigma_2}$ and entries $\ell_1, \ \ell_2$ in Table \ref{lossmatrix-table} and \ref{feedbackmatrix-table}, we have,
\begin{equation}
\begin{aligned}
 S_{\sigma_1}= S_{\sigma_2}=
  \left[ {\begin{array}{cccccccc}
   1 & 1 & 1 & 1 & 0 & 0 & 0 & 0 \\
   0 & 0 & 0 & 0 & 1 & 1 & 1 & 1 \\
  \end{array} } \right]\\
\ell_1 - \ell_2 = 
 \left[ {\begin{array}{cccccccc}
   0 & 1 & -1 & 0 & 0 & 1 & -1 & 0 \\
  \end{array} } \right] .
\end{aligned}
\end{equation}
It is clear that $\ell_1 - \ell_2 \notin Col(S^{\top}_{\sigma_1})$. Hence, Definition 5 fails to hold.
\end{proof}

\section{Efficient Algorithm for Obtaining Regret}

\subsection{Proof of Lemma~\ref{unbiasedestimator}}
\begin{proof}
We can write $\hat{r}_{t}= \sum_{j=1}^m r_{i_j}(j)e_j$, where $e_j$ is the standard basis vector along coordinate $j$. Then $E_{i_1,\ldots,i_m}(\hat{r}_{t})= \sum_{j=1}^m E_{i_j}(r_{i_j}(j)e_j)= \sum_{j=1}^{m} \sum_{k=1}^{t} \dfrac{r_{k}(j)e_j}{t}$= $r_{1:t}^{avg}$.
\end{proof}

\subsection{Proof of Theorem~\ref{efficientregret}}

\begin{proof}


The proof for the top-$1$ feedback needs a careful look at the analysis of FTPL when we divide time into phases/blocks.

\paragraph{FTPL with blocking.} Instead of top-1 feedback, assume that at each round, after learner reveals his action, the full relevance vector is revealed to the learner. Then an $O(\sqrt{T})$ expected regret for $SumLoss$ can be obtained by applying FTPL (Follow the perturbed leader), in the following manner.

At end of every round t, the full relevance vector generated by the adversary is revealed. The relevance vectors are accumulated as $r_{1:t}= r_{1:t-1}+r_t$, where $r_{1:s}= \sum_{i=1}^s r_i$. A learner's action (permutation) for round $t+1$ is generated by solving $M(r_{1:t}+p_t)$, where $p_t \in [0,\frac{1}{\epsilon}]^m$ (uniform distribution) and $\epsilon$ is algorithmic (randomization) parameter. It should be noted that $M(y)= \underset{\sigma}{\argmin} \ \sigma \cdot y$ is simply sorting of $y$ since $f(\sigma)= \sigma$ is a monotone function as defined in Sec.~\ref{rankingmeasures} . 

The key idea is that FTPL implicitly maintains a distribution over $m!$ actions (permutations) at beginning of each round, by randomly perturbing the scores of only $m$ objects: score of each object is sum of (deterministic) accumulated relevance so far and (random) uniform value from $[0,\frac{1}{\epsilon}]$. Thus, it bypasses having to maintain explicit weight on each of $m!$ arms, which is computationally prohibitive. This key property which introduces efficiency in our algorithm is in contrast to the general algorithms based on exponential weights, which have to maintain explicit weights, based on accumulated history, on each action and randomly select an action based on weights.

Now let us look at a variant of the full information problem. The (known) time horizon $T$ is divided into $K$ blocks, i.e., $\{B_1,\ldots,B_K\}$, of equal size $\floor{T/K}$. Here, $B_i= \{(i-1)(T/K)+1, (i-1)(T/K)+2, (i-1)T/K +3, \ldots, i (T/K)\}$. While operating in a block, the relevance vectors within the block are accumulated, but not used to generate learner's actions like in the full information version. Assume at the start of block $B_i$, there was some relevance vector $r^{i}$. Then at each time point in the block, a fresh $p \in [0,\eta]^m$ is sampled and $M(r^i+p)$ is solved to generate permutation for next time point. At the end of a block, the average of the accumulated relevance vectors ($r^{avg}$) for the block is used for updation, as $r^{i} + r^{avg}$, to get $r^{i+1}$ for the next block. The process is repeated for each block. At the beginning of the first block, $r^{1}=\{0\}^m$. 

Formally, let the FTPL have an implicit distribution $\rho_i$ (over the permutations) at the beginning of block $B_i$. That is $ \rho_i \in \Delta$, where $\Delta$ is the probability simplex over $m!$ actions.  Sampling a permutation using $\rho_i$ at each time point of the block $B_i$ means sampling a fresh $p\in [0,\eta]^m$  at every time point $t$ and solving $M(s_{1:(i-1)} + p)$, where  $s_{1:(i-1)}= \sum_{j=1}^{i-1} s_j$ and $s_j$ is the average of relevance vectors of block $B_j$. 
Note that the distribution $\rho_i$ is a fixed, deterministic function of the vectors $s_1,\ldots,s_{i-1}$.

Since action $\sigma_t$, for $t \in B_k$, is generated according to distribution $\rho_k$ (we will denote this as $\sigma_t \sim \rho_k$), and in block $k$, distribution $\rho_k$ is fixed, we have
\begin{multline*}
E_{\sigma_t \sim \rho_k} [ \sum_{t \in [B_k]} SumLoss(\sigma_t,r_t) ]= \\
\sum_{t \in B_k} \rho_k \cdot [SumLoss(\sigma_1,r_t), \ldots, SumLoss(\sigma_{m!},r_t)] .
\end{multline*}

(dot product between 2 vectors of length $m!$).

Thus, the total expected loss of this variant of the full information problem is:
\begin{align}
\label{eq:regret1}
\notag &  \sum_{t=1}^T E_{\sigma_t\sim\rho_k} [ SumLoss(\sigma_t,r_t) ] = \\
&  \sum_{k=1}^K E_{\sigma_t \sim \rho_k}[ \sum_{t \in B_k} SumLoss(\sigma_t,r_t) ]\\
\notag & = \sum_{k=1}^K \sum_{t \in B_k} \rho_k \cdot [SumLoss(\sigma_1,r_t), \ldots, SumLoss(\sigma_{m!},r_t)]\\
\notag & =  \sum_{k=1}^K \sum_{t \in B_k} \rho_k \cdot [\sigma_1 \cdot r_t, \ldots, \sigma_{m!} \cdot r_t)] \ \text{ By defn. of SumLoss}\\
\notag & = \dfrac{T}{K} \sum_{k=1}^K \rho_k \cdot [\sigma_1 \cdot s_k, \ldots, \sigma_{m!}\cdot s_k]\\
\notag &= \dfrac{T}{K} \sum_{k=1}^K E_{\sigma_k \sim \rho_k}[ SumLoss(\sigma_k,s_k) ]\\
& = \dfrac{T}{K} E_{\sigma_1\sim\rho_1,\ldots,\sigma_K\sim \rho_K} \sum_{k=1}^K SumLoss(\sigma_k,s_k) 
\end{align}
where $s_k = \sum_{t \in B_k} \dfrac{r_t}{T/K}$.
Note that, at end of every round $k \in [K]$, $\rho_k$ is updated to $\rho_{k+1}$ by feeding $s_{1:k}$ to FTPL.
By the regret bound of FTPL, for $K$ rounds of full information problem, with $\epsilon=\sqrt{D/RAK}$, we have:
\begin{equation}
\label{eq:regret2}
\begin{aligned}
 & E_{\sigma_1\sim\rho_1,\ldots,\sigma_K\sim\rho_K} \sum_{k=1}^K SumLoss(\sigma_k,s_k)  \\ 
& \le \underset{\sigma}{\min} \sum_{k=1}^K SumLoss(\sigma ,s_k)  + 2 \sqrt{DRAK}\\
& = \underset{\sigma}{\min} \sum_{k=1}^K \sigma \cdot s_k + 2 \sqrt{DRAK} \\
& = \underset{\sigma}{\min} \sum_{t=1}^T \sigma \cdot \dfrac{r_t}{T/K} + 2 \sqrt{DRAK}\\
\end{aligned}
\end{equation}
where $D,R,A$ are parameters dependent on the loss under consideration, that we will discuss and set later. 

Now, since
$$\underset{\sigma}{\min}\ {\sum_{t=1}^T \sigma \cdot \dfrac{r_t}{T/K}}= \underset{\sigma}{\min}\ \dfrac{1}{T/K} \sum_{t=1}^T SumLoss(\sigma,r_t),$$
combining Eq. \ref{eq:regret1} and Eq. \ref{eq:regret2}, we get:
\begin{equation}
\label{eq:regret3}
\begin{aligned}
&\sum_{t=1}^T E_{\sigma_t \in \rho_k} [ SumLoss(\sigma_t,r_t) ]\\
& \le \underset{\sigma}{\min} \sum_{t=1}^T SumLoss(\sigma,r_t)+ 2 \dfrac{T}{K} \sqrt{DRAK}.
\end{aligned}
\end{equation}

\paragraph{FTPL with blocking and top-1 feedback.}
However, in our top-1 feedback model, the learner does not get to see the full relevance vector at each round.  Thus, we form the unbiased estimator $\hat{s}_k$ of $s_k$, using Lemma~\ref{unbiasedestimator}. That is, at start of each block, we choose $m$ time points uniformly at random, and at those time points, we output a random permutation which places each object, in turn, at top. At the end of the block, we form the relevance vector $\hat{s}_k$ which is the unbiased estimator of $s_k$. Note that using $\hat{s}_k$ instead of true $s_k$ makes the distributions $\rho_k$ themselves random.
But significantly, $\rho_k$ is dependent only on information received upto the beginning of block $k$ and is independent of the information collected in the block.  Thus, for block $k$, we have:
\begin{align*}
& E_{\sigma_t \sim \rho_k({\hat{s}_1,\hat{s}_2,..,\hat{s}_{k-1}})} \sum_{t \in [B_k]} SumLoss(\sigma_t,r_t)\\
& = \dfrac{T}{K} E_{\sigma_k \sim \rho_k({\hat{s}_1,\hat{s}_2,..,\hat{s}_{k-1}})} SumLoss(\sigma_k,s_k) \\
&\quad \text{(From Eq.~\ref{eq:regret1})}\\ 
& = \dfrac{T}{K} E_{\sigma_k\sim \rho_k({\hat{s}_1,\hat{s}_2,..,\hat{s}_{k-1}})} E_{\hat{s}_k}SumLoss(\sigma_k,\hat{s}_k) \\
& \quad (\because \text{SumLoss is linear in }s \text{ and $\hat{s}_k$ is unbiased)} \\
& = \dfrac{T}{K} E_{\hat{s}_k} E_{\sigma_k\sim\rho_k({\hat{s}_1,\hat{s}_2,..,\hat{s}_{k-1}})} SumLoss(\sigma_k,\hat{s}_k) .
\end{align*}
In the last step above, we crucially used the fact that, since random distribution $\rho_k$ is independent of $\hat{s}_k$, the order of expectations is interchangeable. Taking expectation w.r.t.
$\hat{s}_1,\hat{s}_2,..,\hat{s}_{k-1}$, we get,
\begin{align*}
& E_{\hat{s}_1,\ldots, \hat{s}_{k-1}} E_{\sigma_t \sim \rho_k({\hat{s}_1,\hat{s}_2,..,\hat{s}_{k-1}})} \sum_{t \in [B_k]} SumLoss(\sigma_t,r_t) =\\
& \dfrac{T}{K} E_{\hat{s}_1,\ldots, \hat{s}_{k-1}, \hat{s}_k} E_{\sigma_t \sim \rho_k({\hat{s}_1,\hat{s}_2,..,\hat{s}_{k-1}})} SumLoss(\sigma_k,\hat{s}_k) .
\end{align*}

Thus, 
\begin{align*}
& \E \sum_{t=1}^T SumLoss(\sigma_t,r_t) = \E \sum_{k=1}^K \sum_{t \in [B_k]} SumLoss(\sigma_t,r_t)\\
& \sum_{k=1}^K E_{\hat{s}_1,\ldots, \hat{s}_{k-1}} E_{\sigma_t \sim \rho_k({\hat{s}_1,\hat{s}_2,..,\hat{s}_{k-1}})}\sum_{t \in [B_k]} SumLoss(\sigma_t,r_t)\\
&  \dfrac{T}{K} \sum_{k=1}^K E_{\hat{s}_1,\ldots, \hat{s}_{k-1}, \hat{s}_k} E_{\sigma_t \sim \rho_k({\hat{s}_1,\hat{s}_2,..,\hat{s}_{k-1}})} SumLoss(\sigma_k,\hat{s}_k) \\
& = \dfrac{T}{K} E_{\hat{s}_1,\ldots, \hat{s}_K} \sum_{k=1}^K E_{\sigma_t \sim \rho_k({\hat{s}_1,\hat{s}_2,..,\hat{s}_{k-1}})} SumLoss(\sigma_k,\hat{s}_k) 
\end{align*}

Now using Eq.~\ref{eq:regret2}, we can upper bound the last term above as
\begin{align*}
& \le \dfrac{T}{K}\{ E_{\hat{s}_1,\ldots,\hat{s}_K}[\underset{\sigma}{\min} \sum_{k=1}^K \sigma\cdot \hat{s}_k ]+ 2 \sqrt{DRAK}\} \\
& \le\dfrac{T}{K}\{ \underset{\sigma}{\min} \sum_{k=1}^K \sigma\cdot s_k + 2 \sqrt{DRAK}\} \\
&\quad \text{(Jensen's \ Inequality)}\\
& \le \underset{\sigma}{\min} \sum_{t=1}^T \sigma \cdot r_t + 2 \dfrac{T}{K} \sqrt{DRAK}\\
& = \underset{\sigma}{\min} \sum_{t=1}^T SumLoss(\sigma, r_t)+ 2 \dfrac{T}{K} \sqrt{DRAK} .
\end{align*}
However, since in each block $B_k$, $m$ rounds are reserved for exploration, where we do not draw $\sigma_t$ from distribution $\rho_k$, the total expected loss is higher. Exploration leads to an extra regret of $RmK$, where $R$, as has been stated before, is an implicit parameter depending on the loss under consideration. The extra regret is because loss in each of the exploration rounds is at most $R$ and there are a total of $mK$ exploration rounds over all $K$ blocks. Thus, overall regret is larger by $RmK$ giving us:
\begin{align*}
&  E\left[ \sum_{t=1}^T SumLoss(\sigma_t,r_t) \right]- \underset{\sigma}{\min} \sum_{t=1}^T SumLoss(\sigma, r_t) \\
& \le RmK + 2\dfrac{T}{K} \sqrt{DRAK} .
\end{align*}
Now we optimize over $K$ and set $K= (DA/R)^{1/3} (T/m)^{2/3}$, to get:
\begin{equation}
\begin{aligned}
\label{eq:mainforregret}
& E\left[\sum_{t=1}^T SumLoss(\sigma_t,r_t)\right] \le   \underset{\sigma}{\min}\sum_{t=1}^T SumLoss(\sigma,r_t) \\
& +  O(m^{1/3}R^{2/3}(DA)^{1/3}T^{2/3})
\end{aligned}
\end{equation}
Now, we recall the definitions of $D$, $R$ and $A$ from \cite{kalai2005}: $D$ is an upper bound on the $\ell_1$ norm of vectors in learner's action space, $R$ is an upper bound on the dot product of vectors in learner's and adversary's action space, and $A$ is an upper bound on the $\ell_1$ norm on vectors in adversary's action space. Thus, for $SumLoss$, we have
\begin{align*}
D &= \sum_{i=1}^m \sigma(i)= O(m^2), \\
R &= \sum_{i=1}^m \sigma(i) r(i) = O(m^2), \\
A &= \sum_{i=1}^m r(i)= O(m) .
\end{align*}
Plugging in these values gives us Theorem~\ref{efficientregret}.
\end{proof}

\section{Regret Bounds for DCG and Prec@k}

We deal with DCG first followed by Prec@k.

\subsection{Extension of Results of SumLoss to DCG}
We give pointers in the direction of proving the following results: a) Local observability condition fails to hold for DCG, b)The efficient algorithm of Sec.\ref{efficientalgorithm} applies to DCG, with regret of $O(T^{2/3})$. Thus, the minimax regret of DCG is $\Theta(T^{2/3})$. All results are applicable to non-binary relevance vectors. The application of Algorithm~\ref{alg:top-1} allows us to skip the proof of global observability, which is complicated for non-binary relevance vectors.

Let adversary be able to choose $r \in \{0,1,\ldots,n\}^m$. Then, from definition of $DCG$ in Sec.\ref{rankingmeasures} , it is clear DCG=$f(\sigma) \cdot g(r)$.  $f(\sigma)$ and $g(r)$ has already been defined for DCG. Both are composed of $m$ copies of univariate, monotonic, scalar valued function, where for $f(\cdot)$, it is monotonically decreasing whereas for $g(\cdot)$, it is increasing.

With slight abuse of notations, the loss matrix $L$ implicitly means gain matrix, where entry in cell $\{i,j\}$ of $L$ is $f(\sigma_i) \cdot g(r_j)$. The feedback matrix $H$ remains the same. In Definition 1, learner action $i$ is optimal if $\ell_i \cdot p \ge \ell_j \cdot p, \ \forall j \neq i$.

In Definition 2, the maximum number of distinct elements that can be in a row of $H$ is $n+1$. The signal matrix now becomes $S_i \in \{0,1\}^{(n+1) \times 2^m}$, where $(S_i)_{k,\ell}= \mathbbm{1}(H_{i,\ell}= k-1)$. 

\subsubsection{Local Observability Fails}

Since we are trying to establish a lower bound, it is sufficient to show it for binary relevance vectors. 

In Lemma~\ref{pareto-optimal}, proved for SumLoss, $\ell_i \cdot p$ equates to $f(\sigma) \cdot E_r [r]$. From definition of DCG, and from the structure and properties of $f(\cdot)$, it is clear that $\ell_i \cdot p$ is \emph{maximized} under the same condition, i.e, $E_{r}[r(\sigma_i^{-1}(1)] \ge E_{r}[r(\sigma_i^{-1}(2)] \ge \ldots \ge E_{r}[r(\sigma_i^{-1}(m)]$. Thus, all actions are Pareto-optimal.

Careful observation of Lemma~\ref{neighbor-actions} shows that it is directly applicable to DCG, in light of extension of Lemma~\ref{pareto-optimal} to DCG.

Finally, just like in SumLoss, simple calculations with $m=3$ and $n=1$, in light of Lemma~\ref{pareto-optimal} and \ref{neighbor-actions}, show that local observability condition fails to hold.

We show the calculations:
\begin{equation*}
\begin{aligned}
 S_{\sigma_1}= S_{\sigma_2}=
  \left[ {\begin{array}{cccccccc}
   1 & 1 & 1 & 1 & 0 & 0 & 0 & 0 \\
   0 & 0 & 0 & 0 & 1 & 1 & 1 & 1 \\
  \end{array} } \right]\\
\end{aligned}
\end{equation*}

\begin{equation*}
\begin{aligned}
\ell_{\sigma_1}= & [0, 1/2, 1/\log_2 3, 1/2 + 1/\log_2 3, 1, 3/2, \\
& 1 + 1/\log_2 3, 3/2 + 1/\log_2 3]\\
\ell_{\sigma_2}= & [0, 1/\log_2 3, 1/2, 1/2 + 1/\log_2 3, 1, 1 + 1/\log_2 3,\\
& 3/2, 3/2 + 1/\log_2 3]
\end{aligned}
\end{equation*}

It is clear that $\ell_1 - \ell_2 \notin Col(S^{\top}_{\sigma_1})$. Hence, Definition 5 fails to hold.

\subsubsection{ Implementation of the Efficient Algorithm}

The only change in Algorithm~\ref{alg:top-1} that allows extension to DCG with non-binary relevance is that relevance values will enter into the algorithm via the transformation $g^s(\cdot)$. That is, every component of relevance vector $r$, i.e., $r(i)$, will become $2^{r(i)}-1$. Every operation of Algorithm~\ref{alg:top-1} will occur on the transformed relevance vectors. It is very easy to see that every step in analysis of the algorithm will be valid by just considering the transformed relevance vectors to be some new relevance vectors with magnified relevance values. The fact that $r$ was binary valued in  $SumLoss$ played no role in the analysis of the algorithm or Lemma~\ref{unbiasedestimator}. The properties which allowed the extension was that $g(\cdot)$ is composed of univariate, monotonic, scalar valued functions and $DCG(\sigma,r)$ is a linear function of $f(\sigma)$ and $g(r)$. 

It is also interesting to note that $M(y)= \underset{\sigma}{\argmax}\ f(\sigma) \cdot y= \underset{\sigma}{\argmin}\ \sigma \cdot y$.  Thus, no changes in the algorithm is required, other than simple transformation of relevance values.

\subsubsection{Proof of Theorem~\ref{efficientregretforDCG}}

Following the proof of Theorem~\ref{efficientregret}, modified for DCG, Eq.\ref{eq:mainforregret} gives (for DCG):
\begin{multline*}
E[\sum_{t=1}^T DCG(\sigma_t,r_t)] \ge \underset{\sigma}{\max}\sum_{t=1}^T DCG(\sigma,r_t) \\
 - O(m^{1/3}R^{2/3}(DA)^{1/3}T^{2/3}) .
\end{multline*}

For DCG, $D= \sum_{i=1}^m f^s(\sigma(i))= O(m), R= \sum_{i=1}^m f^s(\sigma(i))g^s(r(i))= O(m(2^n-1)), A= \sum_{i=1}^m g^s(r(i))= O(m(2^n-1))$ and hence the regret is $O((2^n-1)m^{5/3}T^{2/3 })$.


\subsection{Extension of Results of SumLoss to Prec@k}

Since Prec@$k = f(\sigma) \cdot r$, with $f(\cdot)$ having properties enlisted in Sec.~\ref{rankingmeasures}, all results of SumLoss trivially extend to Prec@$k$, except results on local observability. The reason is that while $f(\cdot)$ of SumLoss is strictly monotonic, $f(\cdot)$ of Prec@$k$ is monotonic but not strict. The gain function depends only on the objects in the top-$k$ position of the ranked list, irrespective of the order. A careful analysis shows that Lemma~\ref{neighbor-actions} fails to extend to the case of Prec@$k$. Thus, we cannot define the neighboring action set of the Pareto optimal action pairs, and hence cannot prove or disprove local observability. The structure of neighbors in Prec@$k$ remains an open question. 

However, the non-strict monotonicity of Prec@k is required for solving $M(y)= \underset {\sigma}{\argmax}\ f(\sigma) \cdot y$ efficiently.

\subsubsection{Proof of Theorem \ref{efficientregretforPrec@k}} 

Following the proof of Theorem \ref{efficientregret}, modified for Prec@$k$, Eq.\ref{eq:mainforregret} gives (for Prec@$k$):
\begin{multline*}
E[\sum_{t=1}^T Prec@k(\sigma_t,r_t)] \ge   \underset{\sigma}{\max}\sum_{t=1}^T Prec@k(\sigma,r_t) \\
 - O(m^{1/3}R^{2/3}(DA)^{1/3}T^{2/3}) .
\end{multline*}

For Prec@$k$, $D= \sum_{i=1}^k f^s(\sigma(i))= O(k), R= \sum_{i=1}^m f^s(\sigma(i))g^s(r(i))= O(k), A= \sum_{i=1}^m r(i)= O(m)$ and hence the regret is $O(km^{2/3}T^{2/3})$.

\section{Non-existence of Sublinear Regret Bounds for NDCG, MAP and AUC}

We show via simple calculations that for the case $m=3$, global observability condition fails to hold for NDCG, when relevance vectors are restricted to take binary values. 
The intuition behind failure to satisfy global observability condition is that the $NDCG(\sigma,r)= f(\sigma) \cdot g(r)$, where where $g(r)= r/ Z(r)$ (See Sec.\ref{rankingmeasures} ). Thus, $g(\cdot)$ cannot be by univariate, scalar valued functions. This makes it impossible to write the difference between two rows as linear combination of columns of (transposed) signal matrices.

\subsection{Proof of Lemma \ref{globalfailsfornormalized}}

\begin{proof}
We will first consider NDCG and then, MAP and AUC.

{\bf NDCG}:

The first and last row of Table \ref{lossmatrix-table}, when calculated for NDCG, are:
\begin{equation*}
\begin{aligned}
\ell_{\sigma_1}= & [1, 1/2, 1/\log_2 3, (1+\log_2 3/2))/(1+\log_2 3), 1, \\
& 3/(2(1+1/\log_2 3)), 1, 1]\\
\ell_{\sigma_6}= & [1, 1, \log_2 2/\log_2 3, 1, 1/2, 3/(2(1+1/\log_2 3)), \\
& (1+(\log_2 3)/2))/(1+\log_2 3), 1]
\end{aligned}
\end{equation*}

We remind once again that NDCG is a gain function, as opposed to SumLoss. 

The difference between the two vectors is:
\begin{equation*}
\begin{aligned}
\ell_{\sigma_1} - \ell_{\sigma_6}=& [0, -1/2, 0, -\log_2 3/(2(1+\log_2 3)), \\
&1/2, 0, \log_2 3/(2(1+\log_2 3)), 0].
\end{aligned}
\end{equation*} 
 
The signal matrices are same as SumLoss:

\begin{equation*}
\begin{aligned}
 S_{\sigma_1}= S_{\sigma_2}=
  \left[ {\begin{array}{cccccccc}
   1 & 1 & 1 & 1 & 0 & 0 & 0 & 0 \\
   0 & 0 & 0 & 0 & 1 & 1 & 1 & 1 \\
  \end{array} } \right]\\
\\
S_{\sigma_3}= S_{\sigma_5}=
  \left[ {\begin{array}{cccccccc}
   1 & 1 & 0 & 0 & 1 & 1 & 0 & 0 \\
   0 & 0 & 1 & 1 & 0 & 0 & 1 & 1 \\
  \end{array} } \right]\\
\\
S_{\sigma_4}= S_{\sigma_6}=
  \left[ {\begin{array}{cccccccc}
   1 & 0 & 1 & 0 & 1 & 0 & 1 & 0 \\
   0 & 1 & 0 & 1 & 0 & 1 & 0 & 1 \\
  \end{array} } \right]\\
\end{aligned}
\end{equation*}

It can now be easily checked that $\ell_{\sigma_1} - \ell_{\sigma_6}$ does not lie in the (combined) column span of the (transposed) signal matrices.

We show similar calculations for MAP and AUC.

{\bf MAP}:

We once again take $m=3$. The first and last row of Table \ref{lossmatrix-table}, when calculated for MAP, is:
\begin{equation*}
\begin{aligned}
&\ell_{\sigma_1}= [1, 1/3, 1/2, 7/12, 1, 5/6, 1, 1]\\
&\ell_{\sigma_6}= [1, 1, 1/2, 1, 1/3, 5/6, 7/12, 1]
\end{aligned}
\end{equation*}

Like NDCG, MAP is also a gain function.

The difference between the two vectors is:
\begin{equation*}
\ell_{\sigma_1} - \ell_{\sigma_6}= [0, -2/3, 0, -5/12, 2/3, 0, 5/12, 0].
\end{equation*} 
 
The signal matrices are same as in the SumLoss case:
\begin{equation*}
\begin{aligned}
 S_{\sigma_1}= S_{\sigma_2}=
  \left[ {\begin{array}{cccccccc}
   1 & 1 & 1 & 1 & 0 & 0 & 0 & 0 \\
   0 & 0 & 0 & 0 & 1 & 1 & 1 & 1 \\
  \end{array} } \right]\\
\\
S_{\sigma_3}= S_{\sigma_5}=
  \left[ {\begin{array}{cccccccc}
   1 & 1 & 0 & 0 & 1 & 1 & 0 & 0 \\
   0 & 0 & 1 & 1 & 0 & 0 & 1 & 1 \\
  \end{array} } \right]\\
\\
S_{\sigma_4}= S_{\sigma_6}=
  \left[ {\begin{array}{cccccccc}
   1 & 0 & 1 & 0 & 1 & 0 & 1 & 0 \\
   0 & 1 & 0 & 1 & 0 & 1 & 0 & 1 \\
  \end{array} } \right]\\
\end{aligned}
\end{equation*}

It can now be easily checked that $\ell_{\sigma_1} - \ell_{\sigma_6}$ does not lie in the (combined) column span of the (transposed) signal matrices.
\end{proof}

{\bf AUC}:

For AUC, we will show the calculations for $m=4$. This is because global observability does hold with $m=3$, as the normalizing factors for all relevance vectors with non-trivial mixture of $0$ and $1$ are same (i.e, when relevance vector has 1 irrelevant and 2 relevant objects, and 1 relevant and 2 irrelevant objects, the normalizing factors are same).  The normalizing factor changes from $m=4$ onwards; hence global observability fails.

Table~\ref{lossmatrix-table} will be extended since $m=4$. Instead of illustrating the full table, we point out the important facts about the loss matrix table with $m=4$ for AUC.

The $2^4$ relevance vectors heading the columns are:

$r_1= 0000,\ r_2= 0001,\ r_3=0010,\ r_4=0100,\ r_5=1000,\ r_6=0011,\ r_7=0101,\ r_8=1001,\ r_9=0110,\ r_{10}=1010,\ r_{11}=1100,\ r_{12}=0111,\ r_{13}=1011,\ r_{14}=1101,\ r_{15}=1110,\ r_{16}=1111$.

We will calculate the losses  of 1st and last (24th) action, where $\sigma_1= 1234$ and $\sigma_{24}=4321$.

\begin{equation*}
\begin{aligned}
&\ell_{\sigma_1}= [0, 1, 2/3, 1/3, 0, 1, 3/4, 1/2, 1/2, 1/4, 0, 1, 2/3, 1/3, 0, 0]\\
&\ell_{\sigma_{24}}= [0, 0, 1/3, 2/3, 1, 0, 1/4, 1/2, 1/2, 3/4, 1, 0, 1/3, 2/3, 1, 0]
\end{aligned}
\end{equation*}

AUC, like SumLoss, is a loss function.

The difference between the two vectors is:
\begin{equation*}
\begin{aligned}
& \ell_{\sigma_1} - \ell_{\sigma_{24}}= \\
& [0, 1, 1/3, -1/3, -1, 1, 1/2, 0, 0, -1/2, -1, 1, 1/3, -1/3, -1, 0].
\end{aligned}
\end{equation*} 

The signal matrices for AUC with $m=4$ will be slightly different. This is because there are 24 signal matrices, corresponding to 24 actions. However, groups of 6 actions will share the same signal matrix. For example, all 6 permutations that place object 1 first will have same signal matrix, all 6 permutations that place object 2 first will have same signal matrix, and so on. For simplicity, we denote the signal matrices as $S_1, S_2, S_3, S_4$, where $S_i$ corresponds to signal matrix where object $i$ is placed at top. We have:

\begin{equation*}
\begin{aligned}
 S_1=
  \left[ {\begin{array}{cccccccccccccccc}
   1 & 1 & 1 & 1 & 0 & 1 & 1 & 0 & 1 & 0 & 0 & 1 & 0 & 0 & 0 & 0 \\
   0 & 0 & 0 & 0 & 1 & 0 & 0 & 1 & 0 & 1 & 1 & 0 & 1 & 1 & 1 & 1 \\
  \end{array} } \right]\\
\\
 S_2=
  \left[ {\begin{array}{cccccccccccccccc}
   1 & 1 & 1 & 0 & 1 & 1 & 0 & 1 & 0 & 1 & 0 & 0 & 1 & 0 & 0 & 0 \\
   0 & 0 & 0 & 1 & 0 & 0 & 1 & 0 & 1 & 0 & 1 & 1 & 0 & 1 & 1 & 1 \\
  \end{array} } \right]\\
\\
 S_3=
  \left[ {\begin{array}{cccccccccccccccc}
   1 & 1 & 0 & 1 & 1 & 0 & 1 & 1 & 0 & 0 & 1 & 0 & 0 & 1 & 0 & 0 \\
   0 & 0 & 1 & 0 & 0 & 1 & 0 & 0 & 1 & 1 & 0 & 1 & 1 & 0 & 1 & 1 \\
  \end{array} } \right]\\
\\
 S_4=
  \left[ {\begin{array}{cccccccccccccccc}
   1 & 0 & 1 & 1 & 1 & 0 & 0 & 0 & 1 & 1 & 1 & 0 & 0 & 0 & 1 & 0 \\
   0 & 1 & 0 & 0 & 0 & 1 & 1 & 1 & 0 & 0 & 0 & 1 & 1 & 1 & 0 & 1 \\
  \end{array} } \right]\\
\end{aligned}
\end{equation*}

It can now be easily checked that $\ell_{\sigma_1} - \ell_{\sigma_{24}}$ does not lie in the (combined) column span of transposes of $S_{1}, S_{2}, S_{3}, S_4$.

\end{document}